\documentclass{article}

\usepackage{url}
\usepackage{color}
\usepackage{xspace}
\usepackage{subfigure,}  % inserted by Una-May
	\usepackage{dsfont}

\usepackage[cmex10]{amsmath}
\usepackage{graphicx}
\usepackage{amssymb, amsthm}

\newtheorem{theorem}{Theorem}
\newtheorem{lemma}{Lemma}

\newtheorem{definition}{Definition}
\newcommand{\ignore}[1]{}
\newtheorem{algorithm}{Algorithm}
\newtheorem{corollary}{Corollary}

\newcommand{\Tlopt}{$T_{lopt}$\xspace}

\newcommand{\ie}{i.\,e.\xspace}
\newcommand{\eg}{e.\,g.\xspace}

\newcommand{\Tmax}{T_{\max}}
\newcommand{\ORDER}{ORDER\xspace}
\newcommand{\MAJORITY}{MAJORITY\xspace}

\newcommand{\currentNode}{{\tt currentNode}\xspace}
\newcommand{\xsubibar}{$\bar{x}_i$\xspace}
\newcommand{\xsubi}{$x_i$\xspace}
\newcommand{\oneonegp}{(1+1)~GP\xspace}
\newcommand{\oneonegps}{(1+1)~GP*\xspace}
\newcommand{\oneonegpsingle}{(1+1)~GP-single\xspace}
\newcommand{\oneonegpmulti}{(1+1)~GP-multi\xspace}
\newcommand{\oneonegpssingle}{(1+1)~GP*-single\xspace}
\newcommand{\oneonegpsmulti}{(1+1)~GP*-multi\xspace}
\newcommand{\hvlMutateTwoPointOh}{HVL-Mutate$^{\mathbf{\prime}}$\xspace}
\begin{document}
\title{Computational Complexity Analysis of Simple Genetic Programming On Two Problems Modeling Isolated Program Semantics}

\author{Greg Durrett\\
Computer Science and Artificial Intelligence Laboratory\\
Massachusetts Institute of Technology\\
Cambridge, MA, USA
\and Frank Neumann\\
Algorithms and Complexity\\
Max Planck Institute for Informatics\\
 Saarbr{\"u}cken, Germany
\and
Una-May O'Reilly\\
Computer Science and Artificial Intelligence Laboratory\\
Massachusetts Institute of Technology\\
Cambridge, MA, USA
}

%\author[mpi]{Frank Neumann}
%\author[mit]{Una-May O'Reilly}
%\address[mit]{Computer Science and Artificial Intelligence Laboratory,
%Massachusetts Institute of Technology,
%Cambridge, MA, USA}
%\address[mpi]{Algorithms and Complexity, Max Planck Institute for Informatics, %Saarbr{\"u}cken, Germany}

\maketitle
\begin{abstract}
Analyzing the computational complexity of evolutionary algorithms (EAs) for binary search spaces has significantly informed our understanding of EAs in general. With this paper, we start the computational complexity analysis of genetic programming (GP). We set up several simplified GP algorithms and analyze them on two separable model problems, ORDER and MAJORITY, each of which captures a relevant facet of typical GP problems. Both analyses give first rigorous insights into aspects of GP design, highlighting in particular the impact of accepting or rejecting neutral moves and the importance of a local mutation operator.
\end{abstract}

%\category{F.2}{Theory of Computation}{Analysis of algorithms and problem complexity}
%A category including the fourth, optional field follows...
%\category{D.2.8}{Software Engineering TODO}{Metrics}[complexity measures, performance measures]

%\begin{keyword}
%Genetic Programming Theory, Computational Complexity, Hill Climbing
%\end{keyword}

%HVL-Mutate$^{\mathbf{\prime}}$ or HVL-Mutate$^{\prime}$

\newpage
\section{Introduction}

Because of the complexity of genetic programming (GP) variants and the challenging nature of the problems they address, it is arguably impossible in most cases to make formal guarantees about the number of fitness evaluations needed for an algorithm to find an optimal solution. Current theoretical approaches investigate foundational aspects of GP tangential to this goal, such as schema theories, search spaces, bloat and problem difficulty \cite{polietal2010}.  However, in this work, we instead choose to follow the path taken for evolutionary algorithms working on fixed-length binary strings. Initial work on pseudo-Boolean functions illustrated the working principles of simple evolutionary algorithms (see \eg \cite{Mue1992,RudBook,DJWoneone}); subsequently, results have been derived for a wide range of classical combinatorial optimization problems such as shortest paths, maximum matchings or minimum spanning trees~(see \eg \cite{BookNeuWit}). These studies have contributed substantially to our theoretical understanding of evolutionary algorithms for binary representations. Poli et al.~\cite{polietal2010} state, ``we expect to see computational complexity techniques being used to model simpler GP systems, perhaps GP systems based on mutation and stochastic hill-climbing.''  This contribution is one fulfillment of this prediction: its goal is to show a GP variant that identifies optimal solutions in provably low numbers of fitness function evaluations for two much simplified, but still relevant, problems that exhibit a few simple aspects of program structure.

The simple parameterized GP algorithm we analyze can succinctly be described as both a hill climber and a randomized algorithm. It has four parametric instantiations we call \oneonegpsingle, \oneonegpmulti, \oneonegpssingle, and \oneonegpsmulti that differ in the acceptance criterion and the size of the mutation proposed. Initially, a solution is chosen at random. We produce by random mutation exactly one offspring of the current solution, and replace the current solution by this proposal as specified by the acceptance criterion. The algorithm iterates until it finds an optimal solution. This simple form of GP algorithm has historical precedent in very early comparisons between Koza-style genetic programming and GP stochastic iterated hill climbing \cite{OReilly:1994:GPSAHC,OReilly:thesis,OReilly:1996:aigp2}, though it does not include a finite bound on fitness evaluations, random restarts or a limit on how many times mutation will be applied to the current solution. Another simplification of the algorithm is that it uses a genetic operator that is as similar to bit-wise mutation as possible. A single bit-wise mutation is the smallest step possible in an binary EA's search space. Our mutation operator makes the smallest alteration possible to the GP tree while respecting the key properties of the GP tree search space: variable length and hierarchical structure.

%The two model problems we select for our analysis are \ORDER and \MAJORITY, defined exactly as in \cite{goldberg:1998:good}.  We have chosen \ORDER and \MAJORITY because each is a particular isolated model of common program semantics\footnote{In \cite{goldberg:1998:good} the authors refer to a program's ``behavior mechanism''. We are describing the same thing but  we choose to clarify ``behavior mechanism'' by, instead, using ``semantics'' which is a well known concept in the theory of programming languages.} They are neither real world application problems or ad-hoc toy problems intended to demonstrate or match to GP's strength (such as boolean multiplexer for classical GP \cite{koza:1992:book} or lawnmower for GP with automatically defined functions \cite{koza:gp2}). They are each indicative of a simple, factored, desirable property of GP program structure.

The two model problems we select for our analysis are \ORDER and \MAJORITY, defined exactly as in \cite{goldberg:1998:good}.  We have chosen \ORDER and \MAJORITY because they make complexity analysis tractable. They allow fitness function evaluation without explicitly executing the program defined by the GP tree. They are minimally sufficient to capture several key properties of GP, including the existence of multiple optimal solutions but they are not real world application problems. Neither are they ad-hoc toy problems intended to demonstrate GP's strength (such as Boolean multiplexer for classical GP \cite{koza:1992:book} or lawnmower for GP with automatically defined functions \cite{koza:gp2}). Each problem has a simple relation to more realistic GP problems: \ORDER requires correct ordering as in conditional programs and \MAJORITY requires the correct set of solution components.

We proceed as follows: in Section~\ref{sect:definitions}, we formally describe the GP variants and the two problems.  This requires that we first describe program initialization from a primitive set (\ref{sect:ProgramInit}) and our mutation operator which is called \hvlMutateTwoPointOh (\ref{sect:HVLMutateDefinition}).  We then proceed in Sections~\ref{sec:order} and \ref{sec:majority} with our analyses of \ORDER and \MAJORITY in terms of the expected number of fitness evaluations until our algorithms have produced a globally optimal solution for the first time. This is called the expected optimization time of the algorithm. Our results are followed by a discussion in Section~\ref{sec:discussion} and conclusions and future work in Section~\ref{sec:conclusion}.
 
% what was already here
%With this paper, we start the rigorous runtime analysis of genetic programming. We introduce simple genetic programming algorithms that serve as baseline algorithms for gaining theoretical insights into the working principles of genetic programming. Our theoretical investigations deal with two different GP variants of the ONEMAX problem. ONEMAX defined on binary strings is the first problem that has been tackled by rigorous methods in the runtime analysis of evolutionary algorithms~\cite{Mue1992,DJWoneone}. Hence, it seems to be a good starting point to consider problems that have similar features for GP as well.
\section{Definitions}\label{sect:definitions}

\subsection{Program Initialization}\label{sect:ProgramInit}

To use tree-based genetic programming, one must first choose a set of primitives $A$, which contains a set $F$ of functions and a set $L$ of terminals. Each primitive has explicitly defined semantics; for example, a primitive might represent a Boolean condition, a branching statement such as an IF-THEN-ELSE conditional, the value bound to an input variable, or an arithmetic operation. Functions are parameterized. Terminals are either functions with no parameters, i.e.~arity equal to zero, or input variables to the program.

In our derivations, we assume that a GP program is initialized by its parse tree construction. In general, we start with a root node randomly drawn from $A$ and recursively populate the parameters of each function in the tree with subsequent random samples from $A$, until the leaves of the tree are all terminals. Functions constitute the internal nodes of the parse tree, and terminals occupy the leaf nodes. The exact properties of the tree generated by this procedure will not figure into the analysis of the algorithm, so we do not discuss them in depth.

%\begin{figure}[h]
%\begin{enumerate}
%\item Create a root nonterminal $J$.
%\item Initialize a set of ``leaf positions'' $L = \{ J_L, J_R \}$, where $J_L$ and $J_R$ represent
%the left and right children of $J$, respectively.
%\item If $L$ contains fewer than $2n$ leaf positions:
%\begin{enumerate}
%\item Select a random leaf position from $L$ and add a new join node $J'$ there in the tree.
%\item Add $J'_L$ and $J'_R$ to $L$.
%\item Go to 3.
%\end{enumerate}
%\item For each leaf position $p_i$ in $L$:
%\begin{enumerate}
%\item Add at position $p_i$ a terminal selected uniformly at random from the set of all possible
%terminals, i.e. $\{ x_1, ..., x_n, \bar{x}_1, ..., \bar{x}_n \}$.
%\item Remove $p_i$ from $L$.
%\end{enumerate}
%\end{enumerate}
%\caption{Unity expectation initialization for GP trees.}
%\label{fig:unity_expectation_init}
%\end{figure}

%The exact properties of the tree generated by this procedure will not figure into the analysis of the algorithm, so we do not discuss them in depth. Effectively, this is useful for generating a tree with $2n$ terminals, each selected uniformly at random from the set of possible terminals.   

\subsection{\hvlMutateTwoPointOh}\label{sect:HVLMutateDefinition}

The \hvlMutateTwoPointOh operator is an update of O'Reilly's HVL mutation operator (\cite{OReilly:thesis,OReilly:1994:GPSAHC}) and motivated by minimality rather than inspired from a tree-edit distance metric. HVL first selects a node at random in a copy of the current parse tree. Let us term this the \currentNode.  It then, with equiprobability, applies one of three sub-operations: insertion, substitution, or deletion. Insertion takes place above \currentNode: a randomly drawn function from $F$ becomes the parent of \currentNode and its additional parameters are set by drawing randomly from $L$. Substitution changes \currentNode to a randomly drawn function of $F$ with the same arity. Deletion replaces \currentNode with its largest child subtree, which often admits large deletion sub-operations.

The operator we consider here, \hvlMutateTwoPointOh, functions slightly differently, since we restrict it to operate on trees where all functions take two parameters. Rather than choosing a node followed by an operation, we first choose one of the three sub-operations to perform. The operations then proceed as shown in Figure~\ref{fig:hvl_prime_example}. Insertion and substitution are exactly as in HVL; however, deletion only deletes a leaf and its parent to avoid the potentially macroscopic deletion change of HVL that is not in the spirit of bit-flip mutation. This change makes the algorithm more amenable to complexity analysis and specifies an operator that is only as general as our simplified problems require, contrasting with the generality of HVL, where all sub-operations handle primitives of any arity. Nevertheless, both operators respect the nature of GP's search among variable-length candidate solutions because each generates another candidate of potentially different size, structure, and composition.

In our analysis on these particular problems, we make one further simplification of \hvlMutateTwoPointOh: substitution only takes place at the leaves. This is because our two problems only have one generic ``join'' function specified, so performing a substitution anywhere above the leaves is a vacuous mutation. Such operations only constitute one-sixth of all operations, so this change has no impact on any of the runtime bounds we derive.

\begin{figure}[htp]
\centering
\subfigure[Before insertion]
  {\includegraphics[width=1.3in,height=0.95in,trim=10mm 15mm 0mm 0mm]{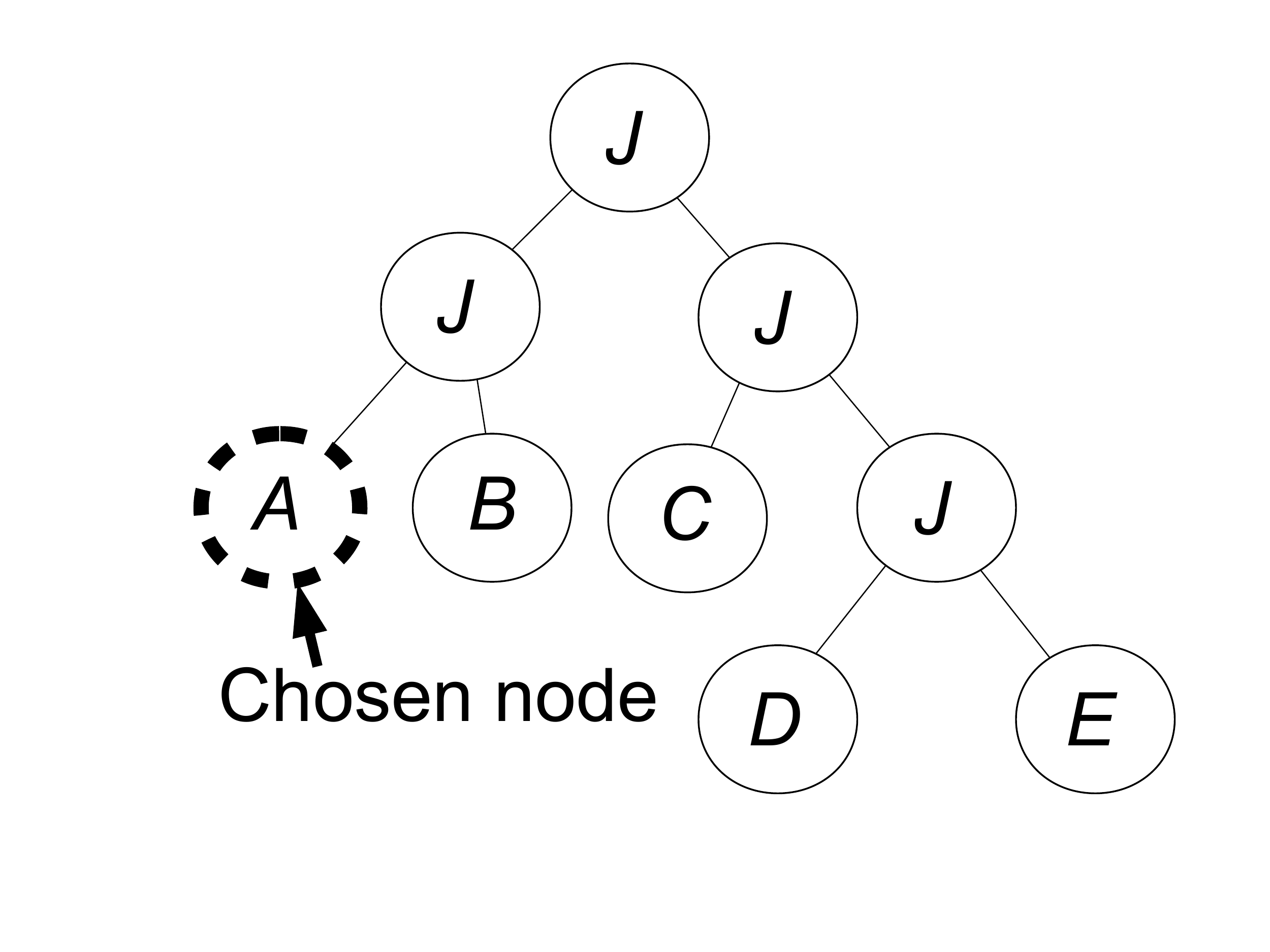}}
\hspace{0.2in}
\subfigure[After insertion]
  {\includegraphics[width=1.4in,height=0.95in,trim=10mm 15mm 0mm 0mm]{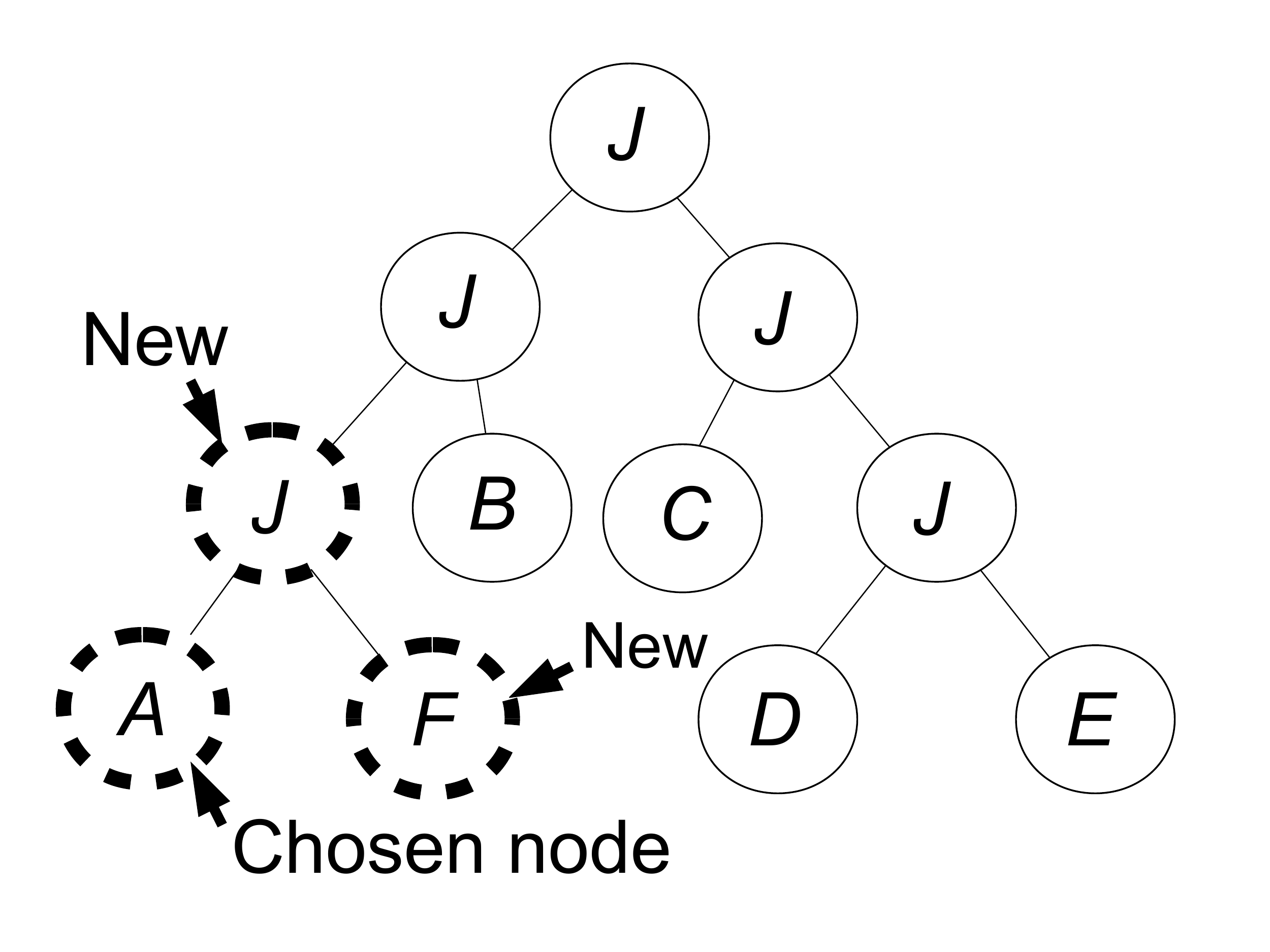}}
\subfigure[Before deletion]
  {\includegraphics[width=1.3in,height=0.95in,trim=10mm 15mm 0mm 0mm]{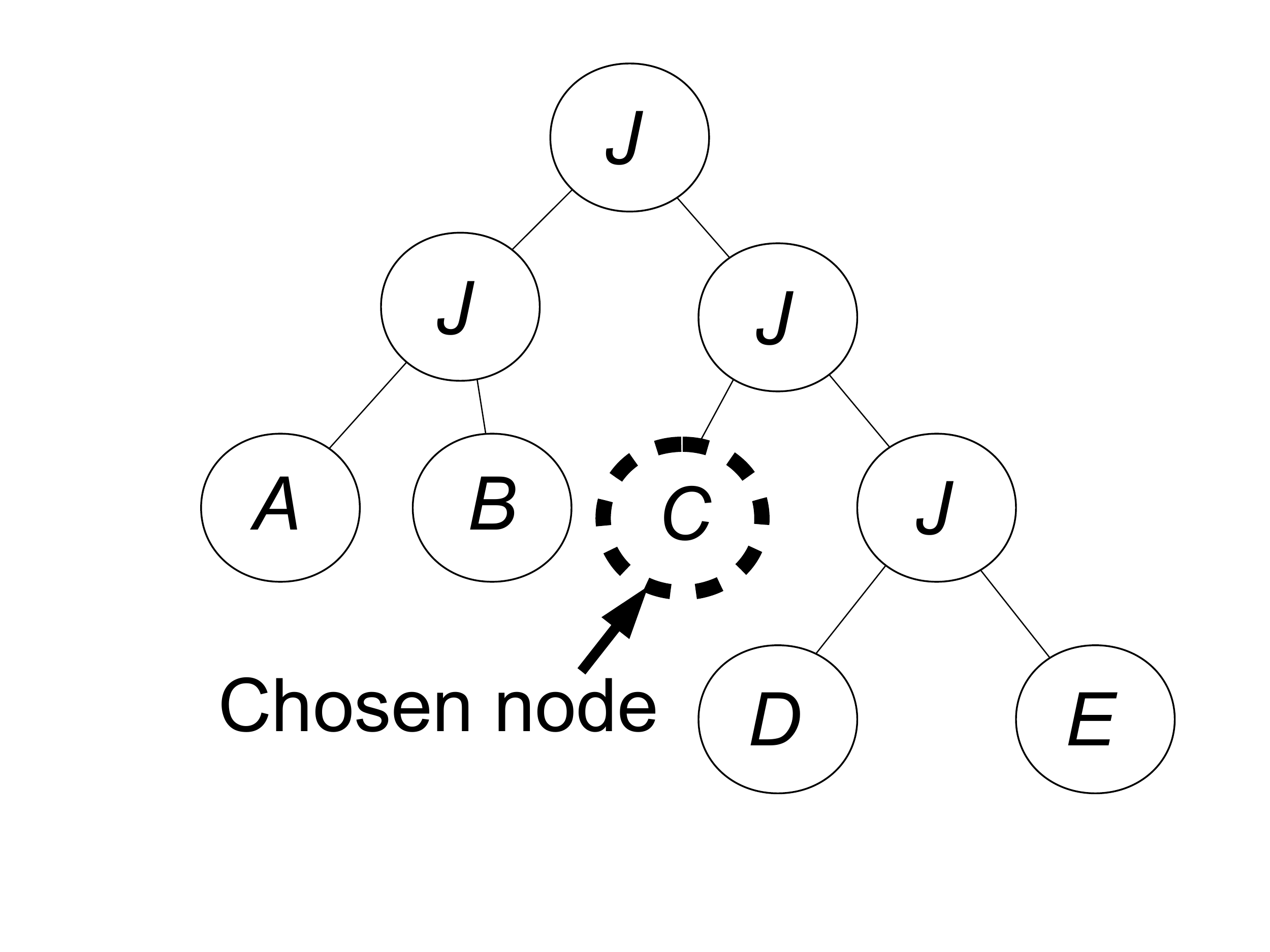}}
\hspace{0.2in}
\subfigure[After deletion]
  {\includegraphics[width=1.4in,height=0.95in,trim=10mm 30mm 0mm 0mm]{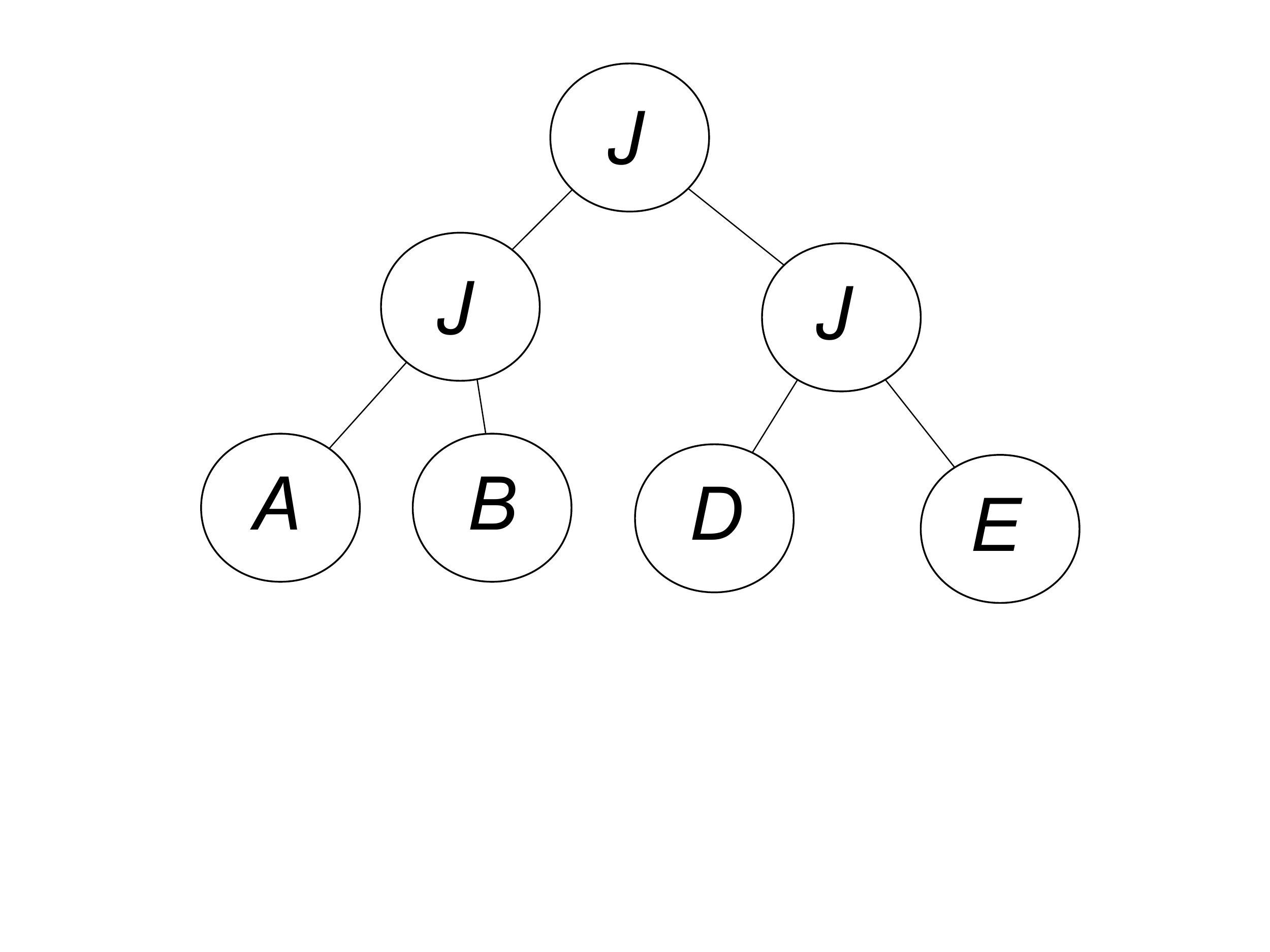}}
\subfigure[Before substitution]
  {\includegraphics[width=1.3in,height=0.95in,trim=10mm 15mm 0mm 0mm]{current-hvl-mutate-prime-2.pdf}}
\hspace{0.2in}
\subfigure[After substitution]
  {\includegraphics[width=1.3in,height=0.95in,trim=10mm 15mm 0mm 0mm]{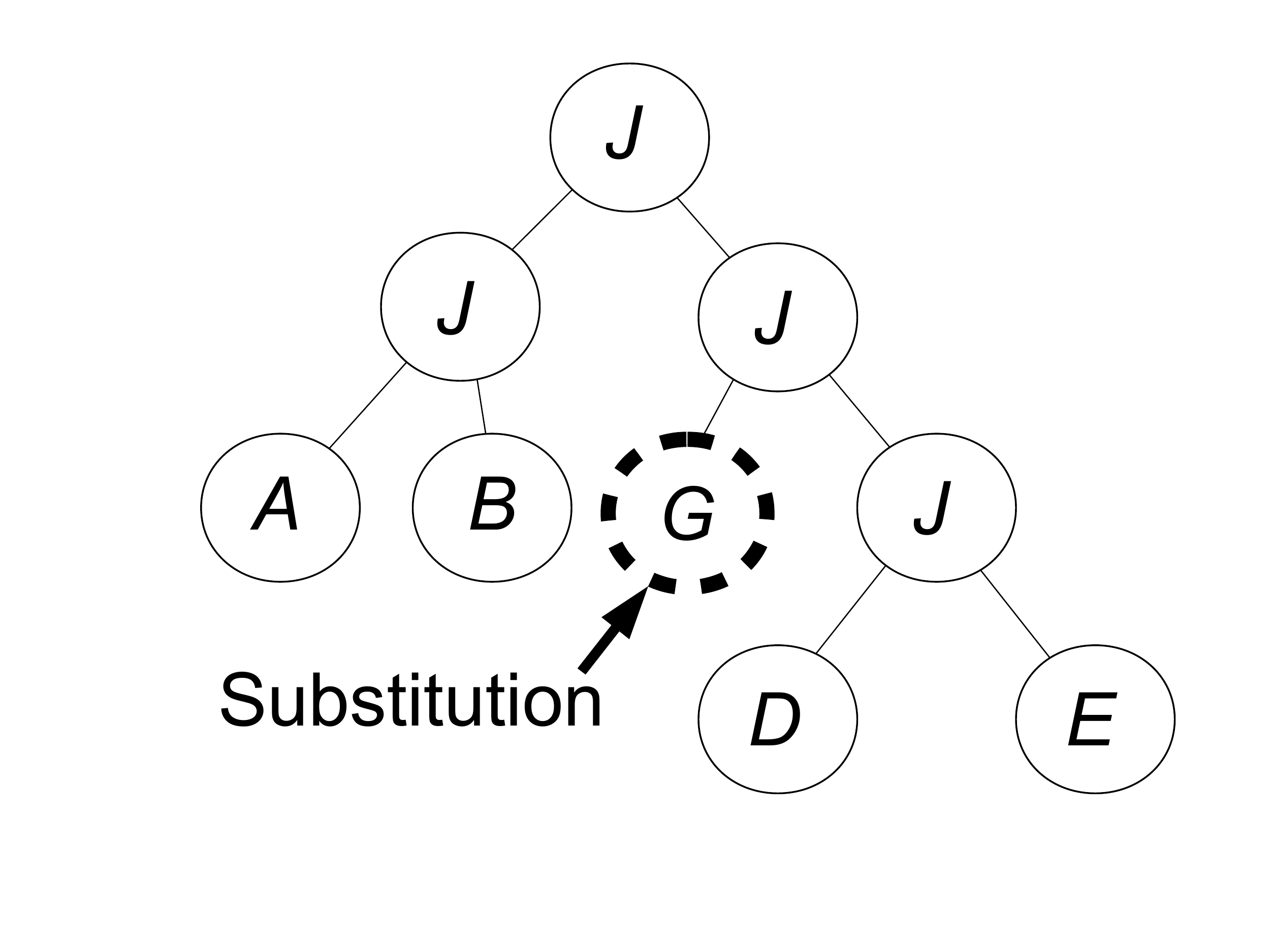}}
\caption{Example of the operators from \hvlMutateTwoPointOh.}
\label{fig:hvl_prime_example}
\end{figure}

%\begin{figure}[hdtp]
%\centering
%\includegraphics[width=1.3in,height=0.95in]{current-hvl-mutate-prime.pdf} %1.0
%\caption{Parse tree before substitution, deletion}\label{fig:hvl_mutate_prime_initial}
%\includegraphics[width=1.3in,height=0.95in]{subst-hvl-mutate-prime.pdf} %1.0
%\caption{Result of substitution}\label{fig:hvl_mutate_prime_subst}
%\includegraphics[width=1.5in,height=0.95in]{delete-hvl-mutate-prime.pdf} %1.0
%\caption{Result of deletion}\label{fig:hvl_mutate_prime_delete}
%\includegraphics[width=1.3in,height=0.95in]{insert-before-hvl-mutate-prime.pdf} %1.0
%\caption{Parse tree before insertion}\label{fig:hvl_mutate_prime_before_insert}
%\includegraphics[width=1.7in,height=1.15in]{insert-after-hvl-mutate-prime.pdf} %1.25
%\caption{Result of insertion}\label{fig:hvl_mutate_prime_after_insert}
%\end{figure}

\subsection{Algorithms}\label{sect:algGP}
\ignore{
\oneonegp and \oneonegps  are each instantiations of the algorithm described in Algorithm~\ref{alg} which is parameterized by an accept function.   The accept function of \oneonegp is  \textit{greater than or equal}, i.e. $\geq$. The accept function of \oneonegps is \textit{greater than}, i.e. $>$. 
}

We define two genetic programming variants called \oneonegp and \oneonegps. Both algorithms work with a population of size one and produce in each iteration one single offspring. \oneonegp is defined in Algorithm~\ref{gp} and accepts an offspring if it is as least as fit as its parent.

%\begin{figure}[h]
\begin{algorithm}[\oneonegp]

\begin{enumerate}
\item[]
\item Choose an initial solution $X$.
\item Set $X':= X$.
\item Mutate $X'$ by applying \hvlMutateTwoPointOh $k$ times. For each application, randomly choose to either substitute, insert, or delete.
\begin{itemize}
\item If substitute, replace a randomly chosen leaf of $X'$ with a new leaf $u \in L$ selected uniformly at random.
\item If insert, randomly choose a node $v$ in $X'$ and select $u \in L$ uniformly at random. Replace $v$ with a join node whose children are $u$ and $v$, with the order of the children chosen randomly.
\item If delete, randomly choose a leaf node $v$ of $X'$, with parent $p$ and sibling $u$. Replace $p$ with $u$ and delete $p$ and $v$.
\end{itemize}
%\vspace{-8pt}
%\begin{description}
%\item[    ] if $v \in F$, replace $v$ by a randomly chosen predecessor. 
%\item[    ] If $v \in L$, delete $v$ and its incoming edge.
%\end{description}
\item If $f(X') \geq f(X)$, set $X:=X'$.
\item Go to 2.
\end{enumerate}
\label{gp}
\end{algorithm}

%\vspace{-12pt}
\oneonegps differs from \oneonegp by accepting only solution that are strict improvements (see Algorithm~\ref{gps}).

\begin{algorithm}[Acceptance for \oneonegps]
\begin{enumerate}
\item[]
\item[4'.] If $f(X') > f(X)$, set $X:=X'$.
\end{enumerate}
\label{gps}
%\caption{A simple randomized GP algorithm parameterized by its accept function in Step 5. %The accept function of \oneonegp is $\geq$.  The accept function of \oneonegps is $>$.}
\end{algorithm}
%\end{figure}

For each of \oneonegp and \oneonegps we consider two further variants which differ in using one  application of  \hvlMutateTwoPointOh (``single'') or in using more than one (``multi''). For \oneonegpsingle and \oneonegpssingle, we set $k=1$, so that we perform one mutation at a time according to the \hvlMutateTwoPointOh framework. For \oneonegpmulti and \oneonegpsmulti, we choose $k = 1+\text{Pois}(1)$, so that the number of mutations at a time varies randomly according to the Poisson distribution.

We will analyze these four algorithms in terms of the expected number of fitness evaluations to produce an optimal solution for the first time. This is called the expected optimization time of the algorithm.

\subsection{The \ORDER{} problem}\label{sect:ORDERdefn}

We consider two separable problems called \ORDER and \MAJORITY that have an independent, additive fitness structure. They both admit multiple solutions on their objective function, which we feel is a key property of a model GP problem because it holds generally for all real GP problems. They also both use the same primitive set:

\begin{itemize}
\item $F:=\{J\}$,  $J$ has arity 2.
\item $L:= \{x_1, \bar{x}_1, \ldots, x_n, \bar{x}_n\}$
\end{itemize}
$x_i$ is the complement of $\bar{x}_i$.  

\ORDER represents problems where the primitive sets include conditional functions, which gives rise to conditional execution paths.  GP classification problems, for example, often employ a numerical comparison function (e.g.~greater than X, less than X, or equal to X). This sort of function has two arguments (subtrees), one which will be executed only when the comparison returns true, the other only when it returns false \cite{koza:1992:book}. Thus, a conditional function results in a conditional execution path, so the GP algorithm must identify and appropriately position the conditional functions to achieve the correct conditional execution behavior for all inputs.

\ORDER is an abstracted simplification of this challenge in that it determines the conditional path execution of a program by tree inspection rather than execution.  Instead of evaluating a condition test and then executing the appropriate condition body explicitly, an \ORDER program's conditional execution path is determined by simply inspecting whether a primitive or its complement occurs first in an in-order leaf parse. Correct programs for the \ORDER problem must express each positive primitive \xsubi before its corresponding complement \xsubibar. This correctness requirement is intended to reflect a property commonly found in the GP solutions to problems where conditional functions are used: there exist multiple solutions, each with different conditional execution paths.

\begin{algorithm}[$f(X)$ for \ORDER]
\begin{enumerate}
\item[]
\item Derive conditional execution path $P$ of X:
\begin{enumerate}
\item[Init: ]$l$ an empty leaf list, $P$ an empty conditional execution path
\item[1.1] Parse X inorder and insert each leaf at the rear of $l$ as it is visited.
\item[1.2] Generate P by parsing $l$ front to rear and adding (``expressing'') a leaf to $P$ only if it or its complement are not yet in P (i.e. have not yet been expressed).  
\end{enumerate}
%\item[Step 2.] Let $count$ be the number of $x_{i}$ in the conditional execution path. $f(X)= count$.
\item $f(X)= |\{x_{i} \in P\}|$.
\end{enumerate}
\end{algorithm}
For example, for a tree X, with (after the inorder parse)  $l = (x_{1}, \bar{x}_4, x_{2}, \bar{x}_1, x_{3}, \bar{x}_6)$,  $P=(x_{1}, \bar{x}_4, x_{2}, x_{3}, \bar{x}_6)$ and $f(X)=3$ because $x_{1}, x_{2}$, $x_{3} \in P$.

%Note that $count$  implies the fitness structure of \ORDER is linear and separable.

\subsection{The \MAJORITY{} problem}\label{sect:MAJORITYdefn}

\MAJORITY is a GP equivalent of the GA OneMax problem \cite{goldberg:1998:good}. \MAJORITY reflects a general (and thus weak) property required of GP solutions: a solution must have correct functionality and no incorrect functionality. Like \ORDER, \MAJORITY is a simplification that uses tree inspection rather than program execution. A correct program in \MAJORITY must exhibit at least as many occurrences of a primitive as of its complement and it must exhibit all the positive primitives of its terminal (leaf) set.  Both the independent sub-solution fitness structure and inspection property of \MAJORITY are necessary to make our analysis tractable. 

%\MAJORITY is designed to abstractly imitate the semantics of executing multiple statements which, in combination, accomplish the required computation and generate the final desired result. While in real programs, the \textit{sequence} of execution of the multiple statements must be correct, \MAJORITY solely  isolates combinative, multiple statement semantics. It is nonetheless arguably relevant because frequently in GP (e.g. in symbolic regression) an offspring program is a fitness improvement over its parent(s) by virtue of having new nodes or leafs that provide more functionality. \MAJORITY also requires incorrect code to be excluded.

\begin{algorithm}[$f(X)$ for \MAJORITY]

\begin{enumerate}
\item[]
\item Derive the combined execution statements $S$ of X:
\begin{enumerate}
\item[Init: ]$l$ an empty leaf list, $S$ is an empty statement list.
\item[1.1] Parse X inorder and insert each leaf at the rear of $l$ as it is visited.
\item[1.2]  For $i \leq n$:  if  $\text{count}(x_{i} \in l) \geq \text{count}(\bar{x}_i \in l)$ and $\text{count}(x_{i} \in l) \geq 1$, add $x_{i}$ to $S$
\end{enumerate}
\item $f(X)= |S|$.
%For $i \leq n$\\
%if  $|x_{i}| \in l \geq |\bar{x_{i}}| \in l$, add $x_{i}$ to $S$\\
%$count_{i}$ = the number of $x_{i}$ in $l$  \\
%$count_{\bar{i}}$ = the number of $\bar{x_{i}}$ in $l$ \\
%if $count_{i}$ $\geq$ $count_{\bar{i}}$, add $x_i$ to $S$ \\
\end{enumerate}
\end{algorithm}

For example, for a tree X, with (after the inorder parse)  $l = (x_{1}, \bar{x}_4, x_{2}, \bar{x}_1,\bar{x}_3, \bar{x}_6, x_{1}, x_{4})$,  $S=(x_{1}, x_{2}, x_{4})$ and $f(X)=3$.

%\subsection{Additional Notation}\label{sect:notation}
%We denote by $\alpha = |F|+ |L|$ the number of distinct primitives.
%Throughout this paper, we denote by $n=|L|/2$ the number of distinct variables in a %program solving \ORDER or \MAJORITY.

% file oneonegp_bounds_proofs.tex  should really be order.tex

\section{Analysis for \ORDER}\label{sec:order}

Here we present bounds for \ORDER on the number of runtime evaluations needed in the execution of \oneonegp and \oneonegps.

We will analyze this GP problem using fitness-based partitions~\cite{DJWoneone}. This requires us to compute the probability of improving the fitness from $k$ to $k+1$ for each value of $k$ between 0 and $n-1$, inclusive. Although our \hvlMutateTwoPointOh operator is complex, we can obtain a lower bound on the probability of making an improvement by considering fitness improvements that arise from insertions. This is described in the following lemma.

\begin{lemma}
\label{lem:order_oneonegp}
Define $p_k$ to be the probability that we perform an insertion that improves the fitness value of the GP tree from $k$ to $k+1$. For the single- and multi-operation variants of \oneonegp and \oneonegps applied to the \ORDER problem, 
$$ p_k = \Omega \left( \frac{(n-k)^2}{n \max\{T, n\}} \right)
$$
where $n$ is the number of variables and $T$ is the number of leaves in the GP tree at the particular iteration.
\end{lemma}

\begin{proof}
When the fitness value is $k$, it must be the case that $k$ different \xsubi appear before their corresponding \xsubibar. To improve the fitness, we must insert one of the $n-k$ unexpressed \xsubi as a leaf that will be visited before a leaf containing the corresponding \xsubibar. Assume for notational ease that these unexpressed \xsubi are indexed by $\{x_1,...,x_{n-k}\}$. Define $A_i$ to be the event that we insert \xsubi into the tree with our mutation operation, and define $B_i$ to be the event that \xsubi is inserted before the corresponding \xsubibar. Given this, we can write out $p_k$ as follows.
$$ p_k = \sum_{i=1}^{n-k} \text{Pr}(A_i) \text{Pr}(B_i | A_i)
$$
With a single operation, the probability of choosing to insert a particular \xsubi is $\frac{1}{6n}$, since we choose to insert with probability $\frac{1}{3}$ and select the variable uniformly at random from the set of $2n$ possible terminals. We can cover the multi-operation case with this analysis as well because the number of operations is sampled according to 1 + Pois(1), so the probability of performing exactly one operation is $\frac{1}{e}$. The probability of $A_i$ is therefore at least $\frac{1}{6en}$, so in both the single- and multi-operation cases, we have
$$ p_k \geq \frac{1}{6e n} \sum_{i=1}^{n-k} \text{Pr}(B_i | A_i)
$$
We need to analyze two cases in computing this sum. Preliminarily, we define $S$ to be the total number of nodes in the GP tree. Note that $S = 2T - 1$, so $S = \Theta(T)$.

Case 1: $T \geq n-k$. We first note that the probability of inserting \xsubi such that it is visited between the $j-1$st leaf and the $j$th leaf in the traversal is at least $\frac{1}{2S}$, since we choose to insert at the $j$th leaf with probability $\frac{1}{S}$ and then add \xsubi as a left child of the new join node with probability $\frac{1}{2}$.

Inserting any of the $n-k$ unexpressed \xsubi before the first leaf in the tree clearly improves the fitness. If we insert at the second position instead, there must still be at least $n-k-1$ choices of \xsubi that yield an improvement: there is only one node that will be traversed before this position in the tree, so there is at most one \xsubibar expressed before this position. We can iterate this argument to see that at the $i$th position, there are still $n-k-i+1$ \xsubi that can be inserted for an improvement to the fitness. By reindexing the \xsubi, we then have that \xsubi can be inserted in at least the first $i$ positions in the tree. Using the fact that the number of leaves $T$ is at least $n-k$, we have that
\begin{align*}
p_k &\geq \frac{1}{6en} \sum_{i=1}^{n-k} \text{Pr}(B_i|A_i)\\
&\geq \frac{1}{6en}\sum_{i=1}^{n-k} \frac{n-k-i+1}{2S} = \frac{1}{6en}\sum_{i=1}^{n-k} \frac{i}{2S}\\
&= \frac{1}{6en}\frac{(n-k)(n-k+1)}{4S} \geq \frac{1}{6en}\frac{(n-k)^2}{\max\{4S,n\}}
\end{align*}
Noting that $S = \Theta(T)$, the asymptotic result follows.

Case 2: $T < n-k$: We can apply the argument of Case 1 up to the $T$th position. After this, we have that for $n-k-T+1$ of the unexpressed \xsubi, the corresponding \xsubibar appears nowhere in the tree, so the probability of an insertion improving the fitness is 1. We also note that $S < 2n$ in this case, allowing us to simplify our expression for $p_k$ as follows.
\begin{align*}
p_k &\geq \frac{1}{6en} \sum_{i=1}^{n-k} \text{Pr}(B_i | A_i)\\
&\geq \frac{1}{6en} \left[ \sum_{i=1}^{T} \frac{i}{2S} \right] + \frac{1}{6en} \sum_{i=T+1}^{n-k} 1\\
&\geq \frac{1}{6en} \left[ \sum_{i=1}^{T} \frac{i}{4n} \right] + \frac{n-k-T}{6en}\\
&= \frac{T(T+1)}{24en^2} + \frac{(n-k-T)(n-k)}{6en(n-k)}
\end{align*}
If $T = \Omega(n-k)$, then we lower-bound $p_k$ using only the first term, which behaves asymptotically in this case as $\Omega\left(\frac{(n-k)^2}{n^2}\right)$. Otherwise, if $T = o(n-k)$, then we use the second term, which then grows according to $\Omega\left(\frac{(n-k)^2}{n^2}\right)$. In either case, because $T$ is less than $n$, we have the desired asymptotic behavior.
\end{proof}

%Because of this constraint, the following configuration represents (up to a re-indexing
%of the \xsubi and \xsubibar) the state of the GP where the probability of improving the fitness via insertion is
%lowest:
%\begin{equation*}
%\bar{x}_1, \bar{x}_2, ..., \bar{x}_{n-k}, ... \text{ [other \xsubi and \xsubibar, including $x_{n-k+1},...,x_n$]}
%\end{equation*}
%[TODO: We might need to expand this as a lemma. The following argument doesn't quite work if $k < n-k$ and there are
%only $k$ elements in the list. Of course, in this case it is even easier to make a valid insertion, but it's still
%a case we're not covering.]
%Let $T$ denote the current number of nodes in the GP tree.
%The probability of inserting \xsubi between elements $j$ and $j+1$ in the traversal is at least $\frac{1}{2T}$, since we
%choose to insert at the $j+1$st leaf with probability $\frac{1}{T}$ and then append \xsubi as a right child with probability
%$\frac{1}{2}$. Note that $x_i$ can be inserted before $\bar{x}_j$ for any $j \in \{1,...,i\}$ in order to improve the
%fitness, so there are $i$ possible locations to insert it.
%
%Importantly, this result holds when \oneonegp is modified to perform multiple operations as described in the theorem. There is
%a constant probability of $\frac{1}{e}$ that we choose to perform exactly one operation, so the probability of performing
%a single insertion that improves the fitness value has the same asymptotic behavior as described in Equation~\ref{eqn:oneonegp_prob_bound}.
%\end{proof}

With this lemma, we can now state the general theorem about the number of fitness evaluations needed for
our \oneonegp variants.

\begin{theorem}
\label{thm:oneonegp}
The expected optimization time of the single- and multi-operation cases of \oneonegp and \oneonegps on \ORDER is $O(nT_{\max})$ in the worst case, where $n$ is the number of \xsubi and $T_{\max}$ denotes the maximal tree size at any stage during the evolution of the algorithm.
\end{theorem}

\begin{proof}
We can apply Lemma~\ref{lem:order_oneonegp} to these algorithms, which implies an asymptotic lower bound on $p_k$, the probability of improving the fitness from $k$ to $k+1$ via an insertion. This also serves as an asymptotic lower bound on the probability of improving the fitness at all, and therefore provides an expected time necessary to improve the fitness, regardless of whether or not we accept neutral moves. In order to determine the total number of evaluations, we must sum the expected number of fitness function evaluations over all intermediate fitness values, from $k = 0$ to $k = n-1$. 

The expected optimization time is therefore upper bounded by 
%Let $F$ denote the total
%number of fitness evaluations required by the algorithm.
\begin{align*}
 \sum_{k=0}^{n-1} \frac{1}{p_k}
&  = \sum_{k=0}^{n-1} O\left(\frac{n\max\{T_k,n\}}{(n-k)^2}\right)\\
& = nT_{\max} \sum_{k=0}^{n-1} O\left(\frac{1}{(n-k)^2} \right)\\
& = O\left(nT_{\max} \sum_{j=1}^\infty \frac{1}{j^2} \right) \\
& = O(nT_{\max}) 
\end{align*}
where the second equality follows from the fact that $T_{\max} \geq T_n \geq n$, and the last equality
follows from the fact that $\sum_{j=1}^\infty \frac{1}{j^2}\leq 2$.
\end{proof}

Note that most GP algorithms explicitly limit the maximum tree size that can be used in an algorithm. Choosing a linear maximum tree size that would still allow us to generate an optimal solution, \ie a tree with at least $n$ leaves, gives an algorithm that solves the \ORDER problem in expected time $O(n^2)$. However, it is also sometimes possible to show that the tree does not get too big during the optimization process. We examine this for \oneonegpssingle and present an upper bound on the expected optimization time.
\begin{corollary}
\label{cor:star-order}
The expected optimization time of \\
\oneonegpssingle on \ORDER is $O(n^2)$ if the tree is initialized with $O(n)$ terminals.
\end{corollary}
\begin{proof}
We note that the maximum value of the fitness is $n$, and the fitness is integer-valued, so if it is strictly increasing with each operation that is accepted, there must be no more than $n$ operations accepted. In the single-operation framework, each operation adds at most two nodes to the tree (if it is an insertion), which means that $\Tmax \leq O(n)+2n = O(n)$ holds during the run of the algorithm.
\end{proof}
The case of \oneonegpsmulti is more difficult to analyze because the expected length of accepted moves may be very different from the expected length of proposed moves, as conditioning on accepting the move will skew the distribution. We conjecture that the bound from Corollary~\ref{cor:star-order} holds in this case as well, but do not present a proof of this.

We also note that because of how our fitness-based partition argument is structured, invoking the average case does not enable us to find any better bounds. Although $k$ will initially be somewhat greater than zero, we will generally still need to improve the fitness $\Theta(n)$ times, so we will have the same asymptotic result.

% end of ignore
\ignore{
Now, we consider \oneonegpsmulti by relating it to \oneonegpssingle. Consider the probability of expressing an additional variable by \oneonegpssingle.
The expected number of additional operations of \oneonegpsmulti is chosen according to the Poisson distribution with $\lambda=1$. Hence, the expected number of additional steps that are carried out in an accepted operation is $1$ which implies that $E[\Tmax]$ is upper bounded by $2n+2n+n=5n$. Using linearity of expectation the expected optimization time is upper bounded $O(nE[\Tmax]) = O(n^2)$.

\end{proof}

We cannot easily show a similar bound for \oneonegp. Suppose that $k$ symbols have been expressed, in a configuration like
$$ x_1, x_2, ..., x_k, \bar{x}_1, \bar{x}_2, ..., \bar{x}_k, ..., \text{other variables}
$$
A delete touching any of the leading $k$ variables will decrease the fitness, so our probability of accepting a delete is $1 - \frac{k}{T}$, where $T$ is the number of leaves in the tree. By analogy with the computation in Lemma~\ref{lem:order_oneonegp}, the probability of accepting an insertion is $1 - \frac{k^2}{nT}$. In this case, the tree is biased towards insertions, particularly if $k$ is small. As the tree gets larger, the probability of increasing the fitness decreases, so one could imagine that in the worst case, the tree size will spiral out of control and the optimization time will be very large. Empirically, this does not seem to be the case [TODO: probably no more room for graphs...], but it demonstrates that bounding $T_{\max}$ via a straightforward argument is difficult.
} % ignore paren
% ignore paren
\section{Analysis for \MAJORITY}
\label{sec:majority}

We next consider the \MAJORITY problem. We start with some preliminary definitions.

\begin{definition}  For a given GP tree, let $c(x_i)$ be the number of $x_i$ variables and $c(\bar{x_i})$ be the number of negated $x_i$ variables present in the tree. For a GP tree representing a solution to the \MAJORITY problem, we define the \emph{deficit} in the $i$th variable by
$$
D_i = c(\bar{x}_i) - c(x_i).
$$
\end{definition}

\begin{definition}
In a GP tree for \MAJORITY, we say that $x_i$ is \emph{expressed} when $D_i \leq 0$ and $c(x_i) > 0$.
\end{definition}

The fitness of a tree $T$ is simply the number of variables that are expressed.
%where $(c_T(x_i) \geq c_T(\bar{x_i}))$ returns $1$ iff this expression holds, and $0$ otherwise.

We note a property of \hvlMutateTwoPointOh for this particular problem that we will make use of later.
\begin{definition}
The substitution decomposability property (SDP) for \MAJORITY states that a substitution is exactly equivalent to a deletion followed by an insertion, which are accepted or rejected as a unit.
\end{definition}
This property follows from the fact that the order of the terminals has no bearing on the fitness of a solution for \MAJORITY.  The variable to be replaced by substitution is selected uniformly at random from the set of leaves of the tree. This is identical to how the variable to be deleted is chosen when using the deletion operator. Substitution then inserts a variable selected uniformly at random from the set of possible terminals, just as the insertion operator does. 

We begin our analysis, in \ref{sect:wcMajority}, with worst case bounds for \oneonegpsingle, \oneonegpssingle, and \oneonegpsmulti. \oneonegpsingle solves the problem quite efficiently, yielding polynomial-time worst-case complexity. However, not accepting neutral moves, as in \oneonegps, results in poor performance: \oneonegpssingle fails to terminate in the worst case, and \oneonegpsmulti requires a number of fitness evalutions exponential in the size of the initial tree. 

In \ref{sect:avgCaseMajority} we derive average case bounds that assume the initial solution tree has $2n$ terminals each selected uniformly at random from $L$. This random tree initialization allows us to bound the maximum deficit in any variable. We show that \oneonegpsingle runs in time $O\left(n T_{\max} \log \log(n)\right)$ in the average case. By contrast, \oneonegpssingle has a constant probability of failing to terminate, and so the expected runtime is infinite.

\subsection{Worst Case Bounds}\label{sect:wcMajority}

\subsubsection{\oneonegpsingle}

We will show here some properties of \oneonegpsingle on \MAJORITY and give a polynomial-time worst-case bound on the performance. Our analysis considers the evolution of the deficits $D_i$ over the course of the algorithm as $n$ parallel random walks. We will show that each positive $D_i$ reaches zero at least as quickly as a balanced random walk, which is the condition for the corresponding \xsubi to be expressed; this, then, gives us the expected number of operations that we are required to perform on a particular variable before it is expressed. Because these arguments do not easily extend to \oneonegpmulti, we omit from this section any treatment of that case.

We begin by establishing the validity of modeling the temporal sequence of each of the $D_i$ as a random walk.

\begin{lemma}
\label{lem:majority_rw}
For \oneonegpsingle on \MAJORITY:

a) The probability of proposing an operation that changes either the number of \xsubi or the number of \xsubibar is $\Omega\left(\frac{1}{n}\right)$.

b) If some \xsubi has a deficit $D_i = d > 0$, we require in expectation $O(dT_{\max})$ proposed operations involving that variable before it is successfully expressed, where $T_{\max}$ is the maximum number of nodes in our GP tree at any timestep of the algorithm.
\end{lemma}
\begin{proof}
a) To see that a particular operation involves \xsubi or \xsubibar with probability $\Omega\left(\frac{1}{n}\right)$, we simply note that the probability of inserting one of the two variables is $\frac{1}{3} \times \frac{2}{2n} = \Omega\left(\frac{1}{n}\right)$.

b) We address each of the three types of operations in turn and show that each is at least as favorable as a balanced random walk in terms of reducing $D_i$ to zero.

Insertion: The probability of inserting \xsubi into the tree is $\frac{1}{6n}$, which is the same as the probability of inserting \xsubibar. Therefore, given that we change $D_i$ with an insertion, we increase it or decrease it in a balanced manner, with probability $\frac{1}{2}$.

Deletion: The probability of a deletion changing $D_i$ is
$$ \frac{c(x_i) + c(\bar{x}_i)}{T}
$$
where $T$ is the size of the GP tree. Given that we do such a deletion, we increase $D_i$ with probability $\frac{c(x_i)}{c(x_i) + c(\bar{x}_i)}$ and decrease it with probability $\frac{c(\bar{x}_i)}{c(x_i) + c(\bar{x}_i)}$, since we pick the variable to delete uniformly at random. However, note that because $D_i > 0$, we have that $c(x_i) < c(\bar{x}_i)$, so the probability of decreasing $D_i$ is greater than the probability of increasing it, so this is actually slightly better than a balanced random walk for the purpose of reducing $D_i$.

Substitution: We now make use of the substitution decomposability property (SDP) defined previously to observe that substitution consists of a deletion followed by an insertion. Therefore, a substitution is simply equivalent to taking one or two steps that tend to reduce $D_i$ with probability at least $\frac{1}{2}$ if $D_i$ is greater than 0.

Consider the $1$-dimensional random walk on the integers $0, 1, \ldots, n$, with $n$ being a reflecting barrier and $0$ being an absorbing barrier. The expected time to reach $0$ when starting at $k$ is $O(kn)$, following the analysis for random walks on undirected graphs carried out in \cite{DBLP:conf/focs/AleliunasKLLR79}. This is precisely the setting we have for our random walk on the $D_i$ if we set $k = d$ and $n = T_{\max}$, so we have that the random walk performed by the $D_i$ reaches zero after at most $O(dT_{\max})$ accepted operations.

We now must address the question of how many operations on the variable must be proposed in order to accept $O(dT_{\max})$ of them. Note that if \xsubi is unexpressed, any insertion or deletion affecting $D_i$ will be accepted, since it cannot possibly decrease the fitness value. The probability of a substitution affecting $D_i$ is, by the SDP, less than or equal to the probability that an insertion affects $D_i$ plus the probability that a deletion affects $D_i$. Therefore, even if every substitution is rejected, we still accept a constant fraction of proposed operations that affect $D_i$, so we only require $O(dT_{\max})$ proposed operations involving \xsubi and \xsubibar in order to have $O(dT_{\max})$ accepted operations.

%It is more difficult to analyze the probability of a substitution being accepted; however, we can show that even if every substitution is rejected, we still accept a constant fraction of the operations involving a particular variable.

%Let $E_S$ denote the event that a substitution changes $D_i$, $E_I$ denote the event that an insertion changes $D_i$, and $E_D$ denote the event that a deletion %changes $D_i$. We have
%\begin{align*}
%\text{Pr}(E_S) &= \text{Pr}(E_I) + \text{Pr}(E_D) - 2 * \text{Pr}(E_I) \text{Pr}(E_D)\\
%&\leq 2 \max \{ \text{Pr}(E_I), \text{Pr}(E_D) \}
%\end{align*}
%using the SDP. The first equality follows from the fact that we want either the insertion or the deletion part of the substitution to change $D_i$, but not both, or they cancel each other out. This implies that the probability of proposing an insertion or deletion changing $D_i$ is a constant fraction of the probability of changing $D_i$ via any operation. Because this constant fraction of operations are guaranteed to be accepted, we only requires $O(d^2)$ proposed operations involving \xsubi and \xsubibar to have $O(d^2)$ accepted operations.

Once $D_i$ reaches zero, we are done and $x_i$ is expressed unless $c(x_i) = c(\bar{x}_i) = 0$. In this case, we clearly cannot do any more deletes, but will either add in an \xsubi or \xsubibar via an insertion or a substitution. Through either operation, we add each variable with probability $\frac{1}{2}$, and therefore successfully express \xsubi with probability $\frac{1}{2}$. In the case where we insert an \xsubibar and increase $D_i$ to one, we again apply our one-dimensional random walk result with $k = 1$ and $n = T_{\max}$ to see that we will return to zero again after only $O(T_{\max})$ additional moves, whereupon either \xsubi is present in the tree and we are done or we can once again attempt to add it. Because we expect to do this procedure only twice before succeeding, it only adds $O(T_{\max})$ steps, and therefore does not change our bound of $O(dT_{\max})$.
\end{proof}

This lemma allows us to establish an upper bound on the number of evaluations for \oneonegp on \MAJORITY given a bound on the largest deficit.

\begin{theorem}
\label{thm:majority_gp_overall}
Let $D = \max_i D_i$ for an instance of \MAJORITY initialized with $T$ terminals drawn from a set of size $2n$ (i.e. terminals $x_1, ..., x_n, \bar{x}_1, ..., \bar{x}_n$ ). Then the expected optimization time of \oneonegpsingle is
$$O(n \log n + DT_{\max} n \log \log n)$$
in the worst case.
\end{theorem}
\begin{proof}
We draw upon a result from Myers and Wilf~\cite{myerswilf2004} about a generalized form of the coupon collector problem. If we have $n$ coupons and wish to acquire at least $k$ of each coupon, we need to draw, in expectation, $O(n \log n + k n \log \log n)$ coupons. When $k$ is at least $\log n$, this is a slight improvement over the naive bound of $O(k n \log n)$ from simply iterating the basic coupon collector problem $k$ times.

Lemma~\ref{lem:majority_rw} tells us two things. Firstly, we have that a proposed operation involves \xsubi or \xsubibar with probability $\Omega(\frac{1}{n})$, so we have a coupon collector problem with slightly perturbed coupon probabilities. Secondly, we find that we need to propose $O(DT_{\max})$ operations involving each terminal in order to express all of our variables. Plugging these into the bound described above yields an asymptotic requirement of $O(n \log n + DT_{\max} n \log \log n)$ fitness function evaluations, as desired.

The only wrinkle in this picture is that the coupon collector assumes that a variable is ``complete'' after a set number of coupons have been collected. While we do not accept moves that reduce the fitness value, an expressed variable $x_i$ could become unexpressed if, during the course of a substitution operation, another variable $x_j$ were simultaneously expressed. However, in this case, we must have had $D_i = 0$ and $D_j = 1$, and we have merely reversed the two, which amounts to a relabeling of the $x_i$ and $x_j$. Because the $D_i$ are the only state variables that we care about in this case, this move effectively does nothing except cause us to make a vacuous move. Because substitutions only make up $\frac{1}{3}$ of all of the proposed moves, such wasted moves can only make up a constant fraction of the total number of moves, and therefore do not change the asymptotics.
\end{proof}

As a corollary of \ref{thm:majority_gp_overall} we can bound the initial $D$ by considering tree initialization.

\begin{corollary}
\label{cor:majority_gp_wc}
When \MAJORITY is initialized with $m = O(n)$ terminals drawn from a set of size $2n$, the expected optimization time of \oneonegpsingle is
$$O(n^2 T_{\max} \log \log n)$$
\end{corollary}
\begin{proof}
This follows from Theorem~\ref{thm:majority_gp_overall} with $D = m$, since the deficit cannot be greater than the number of terminals in the tree.
\end{proof}

%[Please review the following paragraph which I tampered with for clarity ]

We can consider the outcome of the worst case tree initialization both intuitively and experimentally.  We have $D = m$ when all of the leaves consist of instances of one bar variable, say $\bar{x}_1$. Since the $\bar{x}_1$ occupy such a large fraction of the tree, they will frequently be substituted out or deleted.  This suggests that the balanced random walk argument is quite pessimistic given this circumstance. We thus expect that, in practice, this initial condition will be quickly erased. If we put $T_{\max} = O(n)$, we know, from the coupon collector problem, that after an initial phase of $O(n \log n)$ steps, we will have proposed a deletion on every leaf that was initialized in the GP tree. Because deletions are always accepted on negated variables, we will have deleted all of the initial $\bar{x}_1$ variables by the end of this ``erasure'' phase, and only expect to introduce at most $O(\log n)$ of any particular bar variable through insertions and substitutions. This implies that, after this relatively short phase, $D = O(\log n)$, giving an optimization time of $O(n^2 \log n \log \log n)$, a bound very close to the average-case optimization time we present in \ref{sec:maj_ac_nonstar}.

We experimented with this initialization to confirm our intuition. Figure~\ref{fig:majority_oneonegp_wc} shows the results of solving \MAJORITY using  \oneonegpsingle with increasing problem size and trees initialized with $2n$ leaves, each occupied by $\bar{x}_1$.  We tracked the number of fitness evaluations required and, even though we imposed no bound on the tree size, the order of growth relative to $n$ appears to be just barely superlinear. This empirical evidence supports the intuition that the worst-case performance is much closer to the average-case than Corollary~\ref{cor:majority_gp_wc} would suggest.

\begin{figure}[htp]
\centering
\includegraphics[width=60mm, trim=50mm 110mm 50mm 105mm]{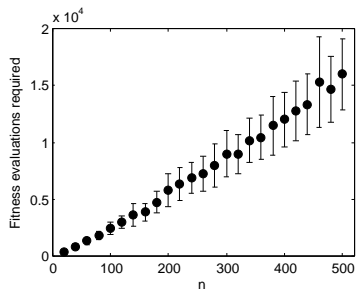}
\caption{Plot of the average optimization time given a ``bad'' initialization of size $2n$ with $D_1 = 2n$ for $x_1$. Fifty trials were used to compute each point. Circles indicate the mean number of fitness function evaluations for each value of $n$, and error bars show the standard deviation of the 50 trials.}
\label{fig:majority_oneonegp_wc}
\end{figure}

\subsubsection{\oneonegps}
\label{sec:maj_wc_star}

Unlike in the case of the \ORDER problem, where accepting or not accepting neutral moves makes no difference in the performance of the algorithm, for \MAJORITY, such a distinction matters tremendously. Intuitively, this behavior arises because there is a notion of ``working towards'' a solution here that is absent from the \ORDER problem. In \ORDER, our analysis relied on our ability to express an \xsubi by simply inserting it as a terminal early enough in the tree, which required only one step. However, in \MAJORITY, if there are $k$ \xsubi and $\ell$ \xsubibar present in the GP tree, at least $\lceil \frac{\ell - k}{2} \rceil$ mutation operations will be required to make up this deficit, all but the last of which will be neutral moves.

Because of the importance of neutral moves, we find that \oneonegpssingle and \oneonegpsmulti perform quite badly. Even when we initialize with a tree with size linear in $n$, the number of terminal symbols, we can demonstrate an initialization where \oneonegpssingle never terminates and \oneonegpsmulti takes an exponential amount of time to do so. Consider the tree \Tlopt which has as leaves the variables 
$$ x_1, x_2, x_3, ..., x_{n-1}, \underbrace{\bar{x_n}, \bar{x}_n ..., \bar{x}_n}_{n+1 \text{ of these}}
$$

\begin{theorem}
\label{thm:majority_gpstar_single_wc}
Let \Tlopt be the initial solution to \MAJORITY. Then the expected optimization time of \oneonegpssingle is infinite.
\end{theorem}
\begin{proof}
It is clear that, with one move, the deficit in $x_n$ can only be changed by at most two. There is a deficit of $n+1$ to make up, which is impossible, therefore \oneonegpssingle will never find its way out of this local optimum.
\end{proof}

\begin{theorem}
\label{thm:majority_gpstar_multi_wc}
Let \Tlopt be the current solution to \MAJORITY. Then the expected optimization time of \oneonegpsmulti is at least exponential in $n$.
\end{theorem}
\begin{proof}
The fitness value of \Tlopt is $n-1$, with $x_1$ through $x_{n-1}$ expressed, so the only way to improve the fitness is to make a move that expresses $x_n$. Therefore, the moves that achieve this are the only moves that will be accepted. We compute the probability of making such a move in this configuration in order to determine the expected time to make such a move.

Note that any mutation operation that successfully improves the fitness must make up for a deficit of $n+1$, which requires at least $\left \lceil \frac{n+1}{2} \right \rceil$ operations, assuming that we, in each case, substitute an $\bar{x}_n$ with an $x_n$. The number of moves per mutation is distributed as $1 + \text{Pois}(1)$, so the Poisson random variable must take a value of at least $\left \lceil \frac{n-1}{2} \right \rceil$. The probability of this when $\lambda = 1$ is given by
$$ \sum_{i = \left \lceil \frac{n-1}{2} \right \rceil}^\infty \frac{e^{-1}}{i!} \leq \frac{1}{\left(\frac{n-1}{2}\right)!} = O\left(\left(\frac{n}{2e}\right)^{-\frac{n}{2}}\right)
$$
by Stirling's formula.

We can take this probability as a (very weak) upper bound on the probability of improving the fitness. Inverting it, we see that the expected number of moves required is $\Omega\left(\left(\frac{n}{2e}\right)^{\frac{n}{2}}\right)$.
\end{proof}

\subsection{Average Case Bounds}\label{sect:avgCaseMajority}

To provide average case bounds we consider a GP tree which is initialized with what we term ``unity expectation'':  it has $2n$ terminals (leaves) each selected uniformly at random from the set of possible terminals. 

% END IGNORE

\subsubsection{\oneonegp}\label{sec:maj_ac_nonstar}

The average case bound follows more or less directly from Theorem~\ref{thm:majority_gp_overall} once a result from the literature is applied to given an expected bound on the maximum initial deficit.

\begin{corollary}
\label{cor:majority_gp_ac}
For \MAJORITY with a terminal set of size $2n$ under unity expectation initialization, the expected optimization time of \oneonegpsingle is $O\left(n T_{\max} \log n\right)$.
\end{corollary}
\begin{proof}
A result from Raab and Steger~\cite{raabsteger1998} tells us that, with probability at least $1 - O\left(\frac{1}{n^k}\right)$ for any integer $k$, no \xsubibar appears more than $O\left(\frac{\log n}{\log \log n}\right)$ times in the GP tree, so $D = O\left(\frac{\log n}{\log \log n}\right)$. Set $k = 2$, so that the probability of having a larger deviation is $O\left(\frac{1}{n^2}\right)$. The worst-case bound of $O\left(n^2 T_{\max} \log \log n\right)$ from Corollary~\ref{cor:majority_gp_wc} ensures that these uncommon cases contribute only an $O(T_{\max} \log \log n)$ term to the expectation. Substituting $D = O\left(\frac{\log n}{\log \log n}\right)$ into the expression in Theorem~\ref{thm:majority_gp_overall} gives us the desired bound for the common case, which is also the overall runtime bound.
\end{proof}

\subsubsection{\oneonegps}
\label{sec:maj_ac_star}

Assuming unity expectation initialization, we can improve on our result from \ref{sec:maj_wc_star} and show \oneonegpssingle has a constant probability of failing to terminate. Our general strategy will be to prove that there is constant probability that, when starting with a deficit of size three in $x_1$, this deficit will be preserved until the fitness is $n-1$. At this point, when all the other variables are expressed, there will remain a gap that cannot be closed in a single step. Such a deficit could disappear over the course of the algorithm because substitution has the ability to shrink the deficit (by removing  $\bar{x}_1$ and replacing it with a $x_i$ in order to express that variable), but this proof shows that there is nonetheless a constant probability of the deficit being preserved.

%
%Assuming unity expectation initialization, we can improve on our result from Section~\ref{sec:maj_wc_star} and show \oneonegpssingle has a constant probability of failing to terminate. Our general strategy will be to prove that, with constant probability, we start with a deficit of size three in $x_1$ and preserve this deficit until the fitness is $n-1$, i.e. all the other variables are expressed, giving us a gap that cannot be made up in a single step. Because substitution has the ability to shrink the deficit (by removing  $\bar{x}_1$ and replacing it with a ${x}_{_1i}$ which is expressed), we will show that there is nonetheless a constant probability of the deficit being preserved.

%old:
%Our general strategy will be to prove 
%that with constant probability, we start with a deÞcit of size 
%three in x1 and preserve this deÞcit until the Þtness is n? 1, 
%giving us a gap that cannot be made up in a single step. 
%Note that such a deÞcit could not be made up even over 
%the entire course of the algorithm if not for the substitution 
%operator: we must address the possibility that the ø 
%x1 that make up this deÞcit will be substituted out as we express 
%other variables. 

First, we establish a lemma about the prevalence of constant-size deficits arising based on our initialization.

\begin{lemma}
\label{lem:majority_gpstar_init}
Suppose we have a $2n$-length instance of the \MAJORITY problem with unity expectation initialization. Let $A_k$ denote the event that $\bar{x}_1$ appears exactly $k$ times without $x_1$ appearing at all, where $k$ is any constant. Then Pr$(A_k)=\Omega(1)$.
\end{lemma}
\begin{proof}
To compute Pr$(A_k)$, we count the number of $2n$-length sequences of terminals for which this is true and divide by the total number of possible sequences. Under $A_k$, we must have $k$ instances of $\bar{x}_1$ and zero instances of $x_1$, so there are ${2n \choose k}$ positions that can be occupied by the $\bar{x}_1$ and the remaining $2n-k$ positions should each be occupied by one of the $2n-2$ elements that are not $x_1$ or $\bar{x}_1$. In total, there are $(2n)^{2n}$ possible $2n$-length sequences of terminals. Combining these facts yields
\begin{align*}
\text{Pr}(A_k) &= \frac{{2n \choose k} (2n-2)^{2n-k}}{(2n)^{2n}}\\
&= \frac{1}{k!} \frac{(2n)!}{(2n-k)!(2n-2)^k} \left(\frac{2n-2}{2n}\right)^{2n}\\
&\geq \frac{1}{k!} \left(\frac{2n-k}{2n-2}\right)^k \left(1 - \frac{1}{n}\right)^{2n}\\
&= \frac{1}{k!} \times \Omega(1) \times \Omega(1)\\
&= \Omega(1)
\end{align*}
assuming that $k$ is a constant.
\end{proof}

Next, we lower-bound the size of the GP tree when running \oneonegpssingle on \MAJORITY. The tree must be large enough so that we are not too likely to substitute out the $\bar{x}_1$ over the course of the algorithm.

\begin{lemma}
\label{lem:majority_gpstar_treesize}
Using \oneonegpsingle on \MAJORITY with any initialization of size $2n$, the size of the GP tree is always greater than $\frac{7n}{6}$ with probability one.
\end{lemma}
\begin{proof}
A deletion can only improve the fitness if we delete some \xsubibar when $c(x_i) = n$ and $c(\bar{x}_i) = n + 1$, with $n$ positive. Such a configuration requires at least three occurrences of \xsubi and \xsubibar in the GP tree, so at most $\frac{2n}{3}$ variables can be present in this fashion initially. Of the at least $\frac{n}{3}$ variables that remain, at most half can be expressed by a deletion, because they must be first put into this configuration during the course of a substitution that expresses some other variable $x_j$. Therefore, we are forced to accept at least $\frac{n}{6}$ insertions or substitutions over the course of the algorithm, giving us an upper bound of $\frac{5n}{6}$ on the number of deletions accepted. This in turn guarantees that our tree always remains larger than $2n - \frac{5n}{6} = \frac{7n}{6}$.
\end{proof}

Finally, we can prove the claim directly.

\begin{theorem}
\label{thm:majority_gpstar_single}
With probability $\Omega(1)$, the optimization time of \oneonegpssingle with unity expectation initialization is infinite on \MAJORITY.
\end{theorem}
\begin{proof}
Lemma~\ref{lem:majority_gpstar_init} tells us that, with a constant probability, we initialize one of the variables, say $x_1$, with $c(x_1) = 0$ and $c(\bar{x}_1) = 3$. We now show that, also with a constant probability, such a deficit is preserved during the course of the expression of at most $n-1$ of the other variables.

We make such an argument by induction. Define the $j$th step of the algorithm as the period after $j$ variables have been expressed, at the end of which we propose the move that expresses the $j+1$st move. Suppose that at the $j$th step, it is true that $c(x_1) = 0$ and $c(\bar{x}_1) = 3$. The $j+1$st variable expressed cannot possibly be $x_1$, since there is no way to make up a deficit of three with a single move. If the move we accept to express the $j+1$st variable is an insertion or a deletion, we preserve our deficit of three and do not change the state of the variable $x_1$ at all, since we must either insert some variable in the set $\{x_2, ..., x_n\}$ or delete some variable in the set $\{\bar{x}_2, ..., \bar{x}_n\}$.

If we express the $j+1$st variable with a substitution, however, it is possible that we might insert an $x_1$ or delete one of the $\bar{x}_1$. An accepted substitution must either replace some variable with a variable in the set $\{x_2,...,x_n\}$ or substitute out some $\{\bar{x}_2, ..., \bar{x}_n\}$. However, a substitution also involves an ``extraneous'' insertion or deletion, by the SDP. If this operation impacts a variable different than the $j+1$st variable we are expressing, it must be an operation that, on its own, would keep the fitness constant. For an extraneous insertion, we note that it is always admissible to insert any of the $n$ symbols in the set $\{x_1, x_2, ..., x_n\}$ without decreasing the fitness. Therefore, the probability of inserting neither $x_1$ nor $\bar{x}_1$ in a neutral or better move is at least $1 - \frac{2}{n}$.

If the extraneous operation is a deletion, we note that it is always possible to delete at least $\frac{T-n}{2}$ terminals, where $T$ is the current tree size. Any variable expressed with $c(x_i) = 1$ and $c(\bar{x}_i) = 0$ cannot be removed, so there might be as many as $n$ terminals forbidden for this reason. For any variable not in this configuration, we have one of two cases. If the variable is unexpressed or is expressed with a deficit less than or equal to -1, any occurrence of $x_i$ or $\bar{x}_i$ can safely be deleted without decreasing the fitness. If the variable is expressed with a deficit of zero, there must be at least as many $\bar{x}_i$ as there are $x_i$, and any of these $\bar{x}_i$ can be safely deleted. Therefore, we set aside at most $n$ ``singleton'' symbols that cannot be deleted, and of those remaining, it must always be acceptable to delete at least half, yielding $\frac{T-n}{2}$.

We therefore preserve our three $\bar{x}_1$ variables with probability at least
$$\frac{\frac{T-n}{2} - 3}{\frac{T-n}{2}} = 1 - \frac{6}{T-n}
$$
We now invoke the result from Lemma~\ref{lem:majority_gpstar_treesize}. Because the size of the tree is at least $\frac{7n}{6}$ at all times, we can lower-bound the probability of preserving the $\bar{x}_1$ as $1 - \frac{36}{n}$.

These situations (extraneous inserts and extraneous deletes) are mutually exclusive, and of the two, the deletes are the more probable to interfere with our $x_1$ setup. Nevertheless, the probability of preserving our deficit of three in $x_1$ from the $j$th step to the $j+1$st step is at least $1 - \frac{O(1)}{n}$. Because there are at most $n-1$ such steps of the algorithm, our overall probability of preserving the deficit is
$$ \left(1 - \frac{O(1)}{n}\right)^{n-1} = \Omega(1)
$$
We have constant probability of initializing with such a deficit, and a constant probability of preserving the deficit, in which case the algorithm never terminates. Therefore, with constant probability, \oneonegpssingle never terminates on \MAJORITY.
\end{proof}

\begin{corollary}
\label{cor:majority_gpstar_single_bound}
Using unity expectation initialization, the expected optimization time of \oneonegpssingle on \MAJORITY  is infinite.
\end{corollary}
\begin{proof}
This follows directly from Theorem~\ref{thm:majority_gpstar_single}  
\end{proof}

While Theorem~\ref{thm:majority_gpstar_single} does demonstrate that the probability of getting stuck in a local optimum is at least a constant, the actual constant yielded by the proof is rather small. However, our proof technique made several very conservative assumptions for simplicity. To investigate further, we tried to solve \MAJORITY with \oneonegpssingle experimentally, observing when we would get stuck in a local minimum. Experimentally, the actual probability of \oneonegpssingle failing to converge to the optimum is actually quite high, as demonstrated by Figure~\ref{fig:majority_oneonegpssingle}.
\begin{figure}[htp]
\centering
\includegraphics[width=60mm,trim=50mm 110mm 50mm 110mm]{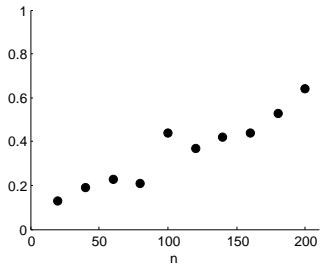}
\caption{Plot showing the probability of \oneonegpssingle failing to terminate on \MAJORITY when the initial solution tree has $2n$ terminals each selected uniformly at random, i.e. with unity expectation initialization. Each probability was determined empirically over the course of 100 simulations for each value of $n$.}
\label{fig:majority_oneonegpssingle}
\end{figure}

However, we do not know how to show a similar result for \oneonegpsmulti. We note that difficult \MAJORITY instances, such as \Tlopt presented in \ref{sec:maj_wc_star}, are exponentially unlikely to occur when the initial solution tree has $2n$ terminals each selected uniformly at  random. From Raab and Steger~\cite{raabsteger1998}, we know that deficits larger than logarithmic occur with exponentially small probability, and in any case large deficits should tend to equalize over the course of the algorithm execution, even if we only accept a linear number of moves. However, if the last unexpressed variable has a deficit of size $k$, we will require at least $\Omega\left(n^{-\frac{k}{2}}\right)$  steps to correctly substitute out enough instances of \xsubibar for \xsubi even in the best case, so unless $k$ can be bounded at a constant, we will have an expected runtime that is superpolynomial.

%$\Omega\left(\frac{1}{n^{\frac{k}{2}}}\right)$ 

\section{Summary and Discussion}\label{sec:discussion}

Table 1 aggregates our expected optimization time results for all algorithm variants and each problem.

\begin{table}[htdp]
\begin{center}
\begin{tabular}{|c|c|c|}\hline
& \multicolumn{2}{|c|}{ORDER} \\ \hline
& \oneonegp & \oneonegps \\ \hline
single & $O(nT_{\max})$ w.c. $\dagger$ & $O(n^2)$ w.c. \\ \hline
multi  & $O(nT_{\max})$ w.c. $\dagger$ & $O(nT_{\max})$ w.c. $\dagger$\\ \hline
\multicolumn{3}{c}{} \\ \hline
&  \multicolumn{2}{|c|}{MAJORITY} \\ \hline
& \oneonegp & \oneonegps \\ \hline 
single & $O(n^2 T_{\max} \log n)$ w.c. $\dagger$ & $\Omega(\infty)$ a.c. \\ 
       & $O(n T_{\max} \log n)$ a.c. & \\ \hline
multi  & ? & $\Omega\left(\left(\frac{n}{2e}\right)^\frac{n}{2}\right)$ w.c.\\ \hline
\end{tabular}
\label{tab:results}
\end{center}
\caption{Results of the computational complexity analysis for our sample problem. We use w.c.~to denote a worst-case bound and a.c.~to denote an average-case bound. The daggers indicate where we conjecture that better bounds exist.}
\end{table}

From the perspective of a GP practitioner, the insights provided by this rigorous analysis may be more valuable than the complexity results themselves. In \ref{sect:algComp} we discuss how our treatment sheds light upon the important but subtle interactions between a problem, the acceptance criterion, and the genetic operator. In \ref{sect:mutateDiscussion}, we discuss the impact of the sub-operations on the mutation operator we considered. In \ref{sect:method}, we address the implications of our design and analysis methodology for practical GP algorithm design. Section \ref{sec:disc_analysis} covers some of our analysis techniques, and finally \ref{sec:conclusion} presents future work avenues and concludes.

\subsection{Accepting Neutral Moves in \ORDER and \MAJORITY}\label{sect:algComp}

It might initially seem immaterial whether or not we accept neutral moves with our genetic operator for \ORDER and \MAJORITY. However, our analysis provides rigorous evidence that the differences in performance between \oneonegp and \oneonegps are substantial for both of these problems. Similar results have already been obtained in the context of evolutionary algorithms for binary representations~\cite{JWplateau}. 
% \ORDER is worst case XX with \oneonegp and YY with \oneonegps.  \MAJORITY is worst case XX %with \oneonegp and YY with \oneonegps.  

\ORDER's focus on condition semantics gives it the property that only the first occurrence of each terminal matters. A large tree makes the probability of improvements smaller because many of the mutations will change variables that have no effect on expression, being sequentially later than earlier occurrences of those same variables. Therefore, not accepting neutral moves helps prevent ``bloat'' and using \oneonegps is significantly advantageous.  \oneonegp's acceptance of neutral moves causes a feedback loop that stimulates growth of the tree: there is a slight bias towards accepting insertions as opposed to deletions, which makes the tree large, which increases the time to find an improvement and results in many neutral insertions, which increase the tree size even more. In general, to solve \ORDER with runtime performance that respects the complexity analysis, the tree must not grow too large, and not accepting neutral moves assures this.

\ignore{In \ORDER, the main factor that can harm our analysis is the tree becoming too large. Because only the first occurrence of each terminal matters, a large tree makes it harder to propose useful moves, since many of our mutations will change variables that are ``covered up'' by earlier occurrences of those some variables. In this case, using the star variant of the operator causes us to accept far fewer moves and gives us a better performance guarantee for the algorithm. For the non-star variant, we enter feedback loop, where the bias towards accepting insertions as opposed to deletions makes the tree large, which increases the time to optimize, which then causes the tree size to drift even higher. Thus, the analysis informs our operator design and we can make a principled decision to use the star variant.} 

%[TODO: should we have some empirical results here to support his? Or should we reign in %the rampant speculation?]

Solving \MAJORITY, we see the opposite effect: \oneonegp very handily beats \oneonegps in terms of expected optimization time. Neutral moves have the effect of balancing both the relative frequency of variables and the number of positive versus negative occurrences. This draws us toward a very favorable average case where every variable is either expressed or very close to being expressed. If neutral moves are not accepted, improvement can frequently stagnate, underscoring the fact that there are large flat regions in the search space. \oneonegp is better equipped to escape these than \oneonegps, so one should, in fact, clearly choose \oneonegp so that there is a guarantee of termination and to avoid the exponential-time worst-case associated with \oneonegpsmulti.

%The fact that not accepting neutral moves results in improvement stagnation so frequently underscores the fact that the fitness landscape has large flat regions, which \oneonegp is better equipped to escape. In this case, one should clearly choose \oneonegp so that there is a guarantee of termination and to avoid the exponential-time worst-cases associated with \oneonegpsmulti.

Overall, these results highlight the fact that, in choosing whether or not to accept neutral moves, one should consider their general effect with respect to both the fitness landscape and growth in tree size. Tying this knowledge into expected optimization time also requires an understanding of the mechanisms by which the fitness increases. We recognize that for \ORDER and \MAJORITY, this is much easier to do rigorously than for more realistic problems. However, perhaps our exercise with \ORDER and \MAJORITY can provide intuitive insight to GP practitioners.

%Overall, these results highlight the fact that, in choosing whether or not to accept neutral moves, one should consider their general effect with respect to both the fitness landscape and growth in tree size. We acknowledge that understanding how this affects the optimization time also requires an understanding of the mechanisms by which the fitness increases, which are harder to make concrete for more realistic problems.

\subsection{Mutation}\label{sect:mutateDiscussion}

Our results also tell us more about our \hvlMutateTwoPointOh framework, and show that it has several interesting properties and behaves quite differently for the two problems.

Interestingly, the analysis for \ORDER, which uses the fitness partition method, only relies on the use of insert. However, we could not run the algorithm with only insertion, because the tree would get very large and the expected time to termination would actually become infinite, in the absence of a strict bound on the tree size. Therefore, deletions are necessary to control the size of the tree, if nothing else. We could, however, envision designing an alternative operator without substitution, whose insertion and deletion probabilities are imbalanced. By choosing these probabilities appropriately, we could prevent the tree size from getting too large (without explicitly bounding it) while still doing as many insertions as possible, thus allowing the algorithm to reach the optimum with the lowest number of fitness evaluations.

For \MAJORITY, the substitution decomposability property indicates that, for this particular problem, the substitution operation is more complex than insertion and deletion and is, in fact, a macro operator which is a combination of the two. A superficial glance at the operator does not necessarily reveal this; in fact, it is tempting to believe that substitution is generally the least complex of the operators, because it most closely resembles a bit-flip in a fixed-length representation and does not change the size or structure of the GP tree at all. However, it could be beneficial to dispense with substitution altogether; this would simplify some of the analysis and achieve the goal of making our mutation operator as ``local'' as possible for the given problem.

Locality is a property that depends on fitness landscape and operator. Here we see the interaction explicitly and reflect upon the influence of the fitness landscape, which itself depends upon the genotype to phenotype mapping. A substitution  makes the same genotypic change because \MAJORITY and \ORDER share the same primitive set. But in \MAJORITY, to a first order approximation, the amortized average change in fitness is larger than in \ORDER because \ORDER's expression mechanism places emphasis on only the front of the parsed leaf list whereas \MAJORITY's depends on the entire set of leaves.

Both problems also reveal a fundamental asymmetry between deletions and insertions. Insertions select uniformly from the set of possible terminals, so each terminal is affected with the same probability, but deletions select uniformly from the set of leaves, so the probability of the operation changing a particular terminal depends on the concentration of that terminal in the tree. In the case of \ORDER, this has ramifications for the evolution of the tree size over time, because insertions end up being less likely to decrease the fitness than deletions, so the tree grows over time. For \MAJORITY, this phenomenon has a positive effect, rather than a negative one: if there are more occurrences of a negated variable than its corresponding positive variable, we will tend to remove those negations with higher probability, and simultaneously balance the relative concentration of each variable.

\subsection{Informing GP Practice}\label{sect:method}

This analysis also prompts one to revisit and review assumptions about the necessity of a population and a crossover operator in GP. It does not imply that they are unnecessary but it explicitly shows, at least, simple circumstances when they are not.  This may advise GP practitioners to assure themselves empirically that a population and a crossover operator are needed (alone or together) when they start their algorithm development. These algorithms and operators are simple and easy to code, yielding a quick solution to the problem while at the same time supporting parallelism and featuring similar efficiency to conventional GP.  Even though the problem at hand is not likely to have the simple problem structure observed here and even though it may require a more sophisticated operator, starting from a provably correct algorithm provides a platform for rationally exploring how to address the separate challenges presented by a harder problem or designing an operator specifically for the problem.

Additionally, this analysis contrasts with conventional GP design practice. Conventionally, GP design proceeds in a very practical manner, but one which is antithetical to theoretical algorithm designers. Rather than derive an algorithm that is provably correct and of efficient complexity, practitioners use biological inspiration, empirical insight and current GP theory. This current theory tries to provide transparent explanations of how GP bloats, how it constructs solutions from schema and how it navigates a fitness landscape with its operators and selection. The resulting heuristic can be expected to generate initial mixed results and require subsequent trial and error to ``perfect'' its use on the problem of interest. There exist some best practices and rules of thumb for robust algorithms but little useful guidance on algorithm customization (via, e.g. genetic operators) for this subsequent design phase. The process finally yields a heuristic which, though a randomized algorithm,  is intractable to analyze post-hoc for correctness or efficiency. One can offer to a user its computational expense, which is the product of population size, generations, and number of runs, as well as some empirical estimate of the likelihood of finding a sufficiently optimal solution on a future problem instance.  An open question is whether the algorithm design methodology taken in this contribution, that of algorithmic theoreticians, could be blended with or complement the method of practice. Our methodology yields a fundamentally different new form of theoretical result for GP: a randomized algorithm of established computational efficiency (which is different from computational expense), that is guaranteed to find a solution. However, our analysis is tractable only because the fitness structure of the model problems is simple, and the algorithms use only a simple hierarchical variable length mutation operator. It is an open question as to whether the first pass application of the simple algorithms and operators on a realistic problem might prove useful for insight or well founded design choices. Forums such as the annual Genetic Programming: From Theory to Practice workshop which encourage explicit interactions among theoreticians and practitioners may encourage the investigation of this question and provide a means of collecting the experiences.

\subsection{Analysis Techniques}\label{sec:disc_analysis}

We also comment briefly on our analysis techniques. The analysis of \ORDER used the method of fitness partitions~\cite{DJWoneone}, a very general method that has found many applications in the complexity analysis of EAs for binary representation. The method was successful because, in \ORDER, we are always only one move away from expressing a particular variable. However, \MAJORITY required different analysis techniques because, unlike in \ORDER, there is not an easily computable probability of improving the fitness given only the fitness value; there is also a crucial dependence on the neutral moves we make. The coupon collector and random walk method considered optimizing all of the variables jointly, and in doing so achieved a bound on performance very close to the theoretical lower bound. This indicates that perhaps more complex fitness functions that are not separable might be admissible to similar styles of analysis.

\subsection{Future Work}\label{sec:conclusion}

We see three main directions for future work in the computational complexity analysis of genetic programming. Obviously the goal of bridging a gap from what exists to practice is daunting. However, modest steps forward may be revealing.  The first extension is to increase the complexity of the genetic operators that are acting on these two problems. Our $1+1$ operators are essentially just stochastic hill climbers, and while understanding of such an optimization technique is valuable in and of itself, real-world GP implementations clearly involve more individuals. From GA theory, there is a precedent of taking $1+1$ analysis of a problem and extending it to $\mu+1$ analysis, where $\mu$ is the size of the population (see \eg \cite{DBLP:journals/ec/Witt06}). This would admit tree-based crossover operators (if they can be shown to be necessary for an efficient optimization time).

The other extension is to consider harder problems. While \ORDER and \MAJORITY each capture a couple very simple properties of a program, both rely upon inspection and neither problem's fitness function takes into account the hierarchical nature of a GP tree, which is of crucial importance for all practical applications of GP. We could extend these problems in several ways to address the latter issue. One could keep the same terminal set and join nodes, but make the fitness function take subtrees into account somehow. Alternatively, one could introduce a new type of join operation and use the fitness function to impose a constraint that forces us to optimize the higher levels of the tree as well, perhaps by giving higher fitness to individuals when this ``join prime'' has regular joins as its children and vice versa. Either of these changes would increase the interest of the problem and make the results more relevant to the way that GP is used in practice. However, any new problem in this form may require new or modified mutation operators to be admissible to analysis. We could also try to extend the difficulty of \MAJORITY by designing a new objective that requires the correct material (and no incorrect material) to be in the right order, in addition to being present.  A model problem which abstracts the semantics of iteration might also provide insights, if tractable. Whether the currently used proof techniques--fitness partitioning, random walks and use of the coupon collector--are sufficient to address more challenging setups is an open question.

\end{document}